\documentclass[journal]{IEEEtran}

\usepackage{cite}

\ifCLASSINFOpdf
  \usepackage[pdftex]{graphicx}
\else
\fi
\usepackage{amsmath}
\usepackage{hyperref}
\usepackage{array}

\ifCLASSOPTIONcompsoc
  \usepackage[caption=false,font=normalsize,labelfont=sf,textfont=sf]{subfig}
\else
\usepackage[caption=false,font=footnotesize]{subfig}
\fi
\usepackage{url}

\usepackage{xspace}

\usepackage{mathptmx}
\usepackage{amsmath}
\usepackage{amssymb}
\usepackage{amsthm}
\usepackage{bm,bbold}
\newtheorem{thm}{Theorem}[section]
\newtheorem{lem}[thm]{Lemma}
\newtheorem{prop}[thm]{Proposition}
\newtheorem*{prop*}{Proposition}

\newtheorem*{cor*}{Corollary}
\theoremstyle{remark}

\usepackage{mathtools}
\usepackage{siunitx}
\usepackage{enumitem}
\usepackage{adjustbox}
\usepackage{stackengine}
\usepackage{url}
\usepackage{microtype}
\usepackage[dvipsnames]{xcolor}

\newcommand{\myCTx}{C_T(x)}

\newcommand{\myPhi}{\Phi}

\newcommand{\myQTx}{Q_T(x)}

\newcommand{\myRxx}{R_{xx}}
\newcommand{\myw}{w}
\newcommand{\myW}{W}
\newcommand{\myU}{U}
\newcommand{\myx}{x}
\newcommand{\myy}{y}
\newcommand{\mylambda}{\lambda}

\newcommand{\RR}{\mathbb{R}}

\definecolor{darkviolet}{rgb}{0.9,0,0.9} 

\DeclareMathOperator{\Trace}{Tr}

\begin{document}
%
\title{Instabilities in Convnets for Raw Audio}
%
%
%


\author{Daniel Haider, Vincent Lostanlen, Martin Ehler, and Peter Balazs
\thanks{D. Haider and P. Balazs are with the Acoustics Research Institute, Austrian Academy of Sciences, Vienna, Austria. V. Lostanlen is with Nantes Université, École Centrale Nantes, CNRS, LS2N, UMR 6004, F-44000 Nantes, France. M. Ehler is with University of Vienna, Faculty of Mathematics, Vienna, Austria.}
}

%
%

\markboth{IEEE Signal Processing Letters,~Vol.~31, pp.~1084-1088, 2024}
{Haider \MakeLowercase{\textit{et al.}}: \title{}}
%



\maketitle

\begin{abstract}
What makes waveform-based deep learning so hard?
Despite numerous attempts at training convolutional neural networks (convnets) for filterbank design, they often fail to outperform hand-crafted baselines.
These baselines are linear time-invariant systems: as such, they can be approximated by convnets with wide receptive fields.
Yet, in practice, gradient-based optimization leads to suboptimal results.
In our article, we approach this problem from the perspective of initialization.
We present a theory of large deviations for the energy response of FIR filterbanks with random Gaussian weights.
We find that deviations worsen for large filters and locally periodic input signals, which are both typical for audio signal processing applications. Numerical simulations align with our theory and suggest that the condition number of a convolutional layer follows a logarithmic scaling law between the number and length of the filters, which is reminiscent of discrete wavelet bases.
\end{abstract}


\IEEEpeerreviewmaketitle

\section{Introduction}

\IEEEPARstart{F}{ilterbanks} are linear time-invariant systems which decompose a signal $\myx$ into $J>1$ subbands. By convolution with filters $(\myw_j)_{j=1,\dots,J}$ the output of a filterbank $\myPhi$ is given by $\left(\myPhi \myx\right)[n,j] = (\myx \ast \myw_j)[n]$.
Filterbanks play a key role in speech and music processing: constant-Q-transforms, third-octave spectrograms, and Gammatone filterbanks are some well-known examples \cite{necciari2018audlets, balazs2017framespsycho, lyon2017human}.
Beyond the case of audio, filterbanks are also used in other domains such as seismology \cite{meier2015gutenberg}, astrophysics \cite{chassande2003learning}, and neuroscience \cite{ang2008filter}.

In deep learning, filterbanks serve as a preprocessing step to signal classification and generation.
In this context, filterbank design is a form of feature engineering.
Yet, in recent years, several authors have proposed to replace feature engineering with feature learning: i.e., to optimize filterbank parameters jointly with the rest of the pipeline \cite{dorfler2020basic,ravanelli2018speaker,zeghidour2021leaf}.
 
So far, prior work on filterbank learning has led to mixed results.
For example, on the TIMIT dataset, using a convolutional neural network (convnet) with 1-D filters on the ``raw waveform'' was found to fare poorly (29.2\% phone error rate or PER) compared to the mel--spectrogram baseline (17.8\% PER) \cite{zeghidour2018learning}.
Interestingly, fixing the convnet weights to form a filterbank on the mel--scale brings the PER to 18.3.\%, and fine-tuning them by gradient descent, to 17.8\%.
Similar findings have been reported with Gammatone filterbanks \cite{sainath2015learning}.

Arguably, such a careful initialization procedure defeats the purpose of deep learning; i.e., sparing the effort of feature engineering.
Furthermore, in some emerging topics of machine listening such as bioacoustics, it would be practically useful to train a filterbank with random initialization to learn something about acoustic events of interest with minimal domain-specific knowledge \cite{hopp2012animal, stowell2022computational}.
Yet, filterbank learning has been outperformed by filterbank design, particularly from a random initialization
\cite{lopez2021exploring,schluter2022efficientleaf,bravo2021bioacoustic}.
Recently, multiresolution neural networks (MuReNN) \cite{lostanlen2023fitting} have circumvented this issue in practice.
In general, however, there are no theoretical results dedicated to filterbank learning specifically.

Prior publications have shown that Lipschitz stability in neural nets is crucial for robustness against adversarial examples \cite{cisse2017parseval, gupta2022lipschitz}. In this article, we allocate this paradigm to the setting of filterbank learning, where we focus on the stability of convnets with 1-D filters at the stage of random Gaussian initialization, interpreted as \textit{random filterbanks}. We hope that this provides valuable insights for understanding their general learning behavior.
While there are methods to compute Lipschitz constants of neural nets deterministically \cite{virmaux2018lipschitz,balan2018lipschitznn}, the originality of our work lies in regarding a random filterbank $\myPhi$ as a non-i.i.d. random matrix to study the behavior of the bounds $A,B$ in the inequality $A \Vert \myx \Vert^2 \leq\Vert \myPhi \myx \Vert^2 \leq B \Vert \myx \Vert^2$ in a probabilistic setting. If $A$ and $B$ are close to each other, this can be interpreted as $\myPhi$ satisfying an energy preservation property with high probability.
With this perspective, we show that natural autocorrelation characteristics of audio signals trigger instabilities in $\myPhi$ with high probability. We also find that the bounds $A,B$ are highly sensitive to the design of the random filterbank, i.e., the number and length of the filters.

\begin{figure}[t]
    \centering
    \includegraphics[width =\linewidth]{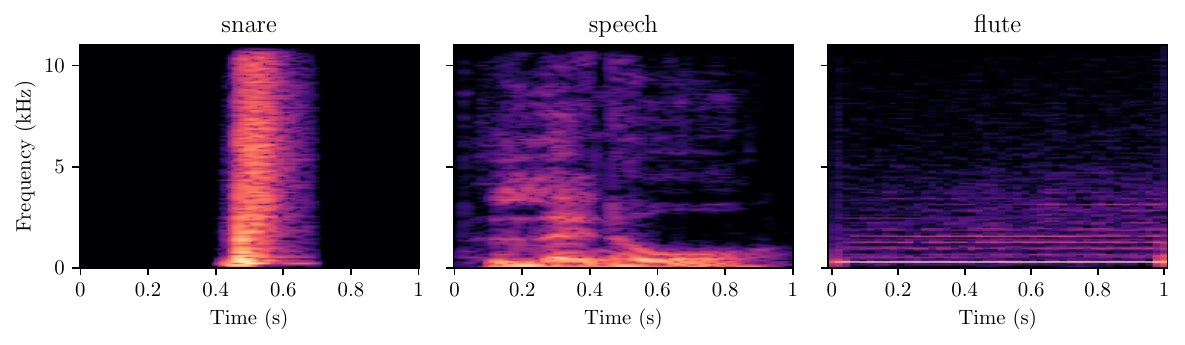}
    \includegraphics[width =\linewidth]{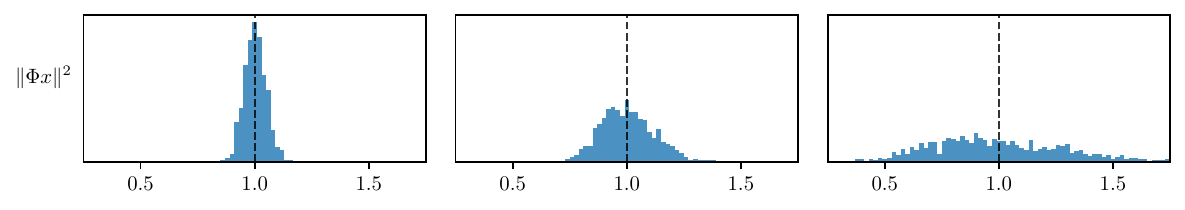}
    \caption{Autocorrelation in the input signal $\myx$ increases the variance of the filterbank response energy $\Vert \myPhi \myx \Vert^2$ across random initializations. We compare audio signals with different autocorrelation profiles. Left to right: Snare (low), speech (medium), and flute (high). Top: Spectrograms of the signals. Bottom: Empirical histogram of $\Vert \myPhi \myx \Vert^2$ for 1000 independent realizations of $\myPhi$.
    }
    \label{fig:spec}
\end{figure}

In Section \ref{sec:x-known}, we
prove explicit formulas for the expected value and variance of the random variable $\Vert \myPhi \myx \Vert^2$ for given input signal $\myx$, and derive upper bounds for the probability of large deviations.
In Section \ref{sec:x-unknown}, we bound the expected values and variances of the optimal stability bounds of $\myPhi$, i.e., $A = \min_{\Vert x \Vert = 1}\Vert \myPhi \myx \Vert^2$ and $B = \max_{\Vert x \Vert = 1}\Vert \myPhi \myx \Vert^2$.
We conclude with an asymptotic analysis of the stability of $\myPhi$ by means of its condition number $\kappa = B/A$.

\section{FIR Filterbank with Random Gaussian Weights}
\label{sec:x-known}
Throughout this article, we use finite circulant convolution of signals $\myx\in \RR^N$ with filters $\myw\in \RR^T$, $T\leq N$, given by 
\begin{equation}\label{eq:conv}
    (\myx \ast \myw)[n] = \sum_{k=0}^{T-1}\myw[k]\myx[(n-k)\;\mathrm{mod}\;N].
\end{equation}
We denote the circular autocorrelation of $\myx$ for $0\leq t < T$ by
\begin{equation}\label{eq:auto}
R_{\myx \myx}(t) = \sum_{k = 0}^{N-1} \myx[k]\myx[(k-t)\;\mathrm{mod}\;N].
\end{equation}

\subsection{Moments of the squared Euclidean norm}

\begin{prop}\label{prop:exp}
    Let $\myx\in \RR^N$ and $\myPhi$ a random filterbank with $J$ i.i.d. filters $\myw_{\!j}\sim \mathcal{N}(0,\sigma^2 I)$ of length $T\leq N$.
    Then expectation and variance of $\Vert\myPhi\myx\Vert^2$ satisfy
    \begin{gather}\label{eq:expPhi}
        \mathbb{E}\left[\Vert \myPhi\myx\Vert^2\right] = J T\sigma^2 \Vert\myx\Vert^2,\\\label{eq:var}
    \mathbb{V}\left[\Vert \myPhi\myx\Vert^2\right] = 2 J \sigma^4 \sum_{\tau=-T}^{T} \big(T-\vert\tau\vert\big) \myRxx(\tau)^2.
    \end{gather}
\end{prop}
We note that \eqref{eq:expPhi} is known for $J=1$ and $T=N$ \cite{ehler2015pre}. Setting $\sigma^2=(JT)^{-1}$ implies $\mathbb{E}\left[\Vert \myPhi\myx\Vert^2\right] = \Vert\myx\Vert^2$.
In other words, if the variance of each parameter $w_j$ scales in inverse proportion with the total number of parameters (i.e., $JT$), then $\myPhi$ satisfies energy preservation on average.
However, it is important to see that the variance of the random variable $\Vert \myPhi\myx\Vert^2$ depends
also on the content of the input $\myx$: specifically, its autocorrelation $\myRxx$.
This is a peculiar property of convnets that does not happen in fully connected layers with random Gaussian initialization \cite{ehler2015pre}. A discussion on this can be found in the appendix.

We note that natural audio signals are often locally periodic and thus highly autocorrelated.
Hence, we interpret Proposition \ref{prop:exp} as follows: untrained convnets are particularly unstable in the presence of vowels in speech or pitched notes in music. Figure \ref{fig:spec} shows this phenomenon for three real-world signals.
\begin{lem}\label{lem:quad}
    Let $\myx\in \RR^N$ and $\myw\in\RR^T$, $T\leq N$.
    The circular convolution of $\myx$ and $\myw$ satisfies $\Vert\myx \ast \myw\Vert^2 = \myw^{\top} \myQTx \myw$, where the entries of the matrix $\myQTx$
    are given by 
    $\myQTx[n,t] = R_{\myx \myx}((t-n)\;\mathrm{mod}\;N)$ for each $0\leq n, t < T$.
    In particular, all diagonal entries of $\myQTx$ are equal to $\|\myx\|^2$.
\end{lem}
The result gets obvious when writing $\myx \ast \myw = C_T(\myx) \myw$ and noting that $Q_T(\myx) = C_T(\myx)^\top C_T(\myx)$.
A detailed proof can be found in the appendix. With this we prove Proposition \ref{prop:exp}.
\begin{proof}[Proof of Proposition \ref{prop:exp}]
Given a filter $\myw_{\!j}$ for $1\leq j\leq J$, we apply Lemma \ref{lem:quad} and use the cyclic property of the trace
\begin{equation}
\Vert\myx \ast \myw_{\!j}\Vert^2 =
\Trace \left(\myw_{\!j}^\top \myQTx \myw_{\!j}\right) =
\Trace \left(\myQTx \myw_{\!j} \myw_{\!j}^\top \right).
\label{eq:norm-x-wf-squared}
\end{equation}
We take the expected value on both sides and recognize the term $\mathbb{E}[\myw_{\!j} \myw_{\!j}^\top]$ as the covariance matrix of $\myw_{\!j}$, i.e., $\sigma^2 I$. Hence:
\begin{equation}\label{eq:trace}
\mathbb{E}\big[\Vert\myx \ast \myw_{\!j}\Vert^2\big] =
\Trace \left(\myQTx \mathbb{E}\left[\myw_{\!j} \myw_{\!j}^\top \right]\right) =
\sigma^2 \Trace \left(\myQTx\right).
\end{equation}
By Lemma \ref{lem:quad}, $\Trace (\myQTx) = T \Vert \myx \Vert^2$, hence $\mathbb{E}[\Vert\myx \ast \myw_{\!j}\Vert^2] = \sigma^2  T \Vert \myx \Vert^2$.
For the variance, we recall Theorem 5.2 from \cite{rencher2008linear}, which states that if $v\sim\mathcal{N}(\mu,\Sigma)$, then for any matrix $M$
\begin{equation}\label{eq:varquadr}
\mathbb{V}\big[v^{\top}M v\big] =
2 \Trace \big(M \Sigma M \Sigma\big)
+ 4 \mu^{\top}M\Sigma M \mu
\end{equation}
We set $v=\myw_{\!j}$, $\mu = 0$, $\Sigma = \sigma^2 I$, and $M = \myQTx$. We obtain:
\begin{align}
\mathbb{V}\big[\Vert\myx \ast \myw_{\!j}\Vert^2\big]
&= 2 \sigma^4 \Trace\left(\myQTx^2\right) \nonumber \\
&= 2 \sigma^4 \sum_{t=0}^{T-1} \sum_{t^\prime=0}^{T-1} \myRxx(t^\prime-t) \myRxx(t-t^\prime) \nonumber \\
&= 2 \sigma^4 \sum_{t=0}^{T-1} \sum_{\tau=-t}^{T-1-t} \myRxx(\tau)^2.
\end{align}
By a combinatorial argument, the double sum above rewrites as $\sum_{\tau=-T}^{T} \big(T-\vert\tau\vert\big) \myRxx(\tau)^2$.
The proof concludes by linearity of the variance, given the independence of the $J$ filters in $\myPhi$.
\end{proof}

\begin{figure}
    \centering
    \includegraphics[width =\linewidth]{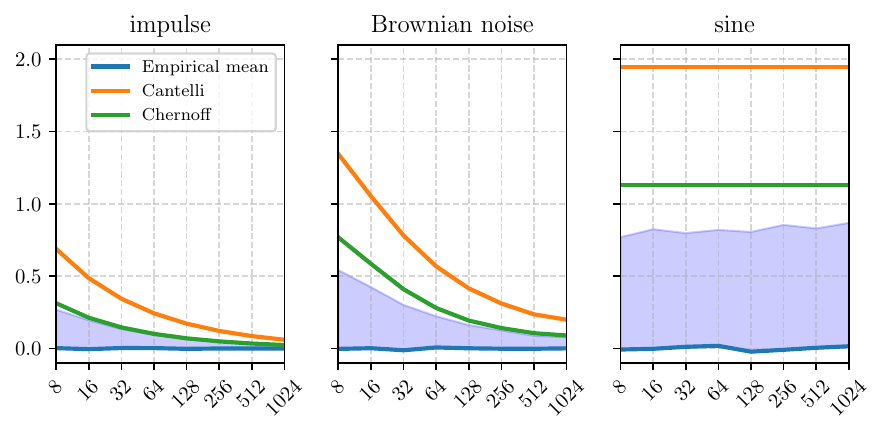}
    \includegraphics[width =\linewidth]{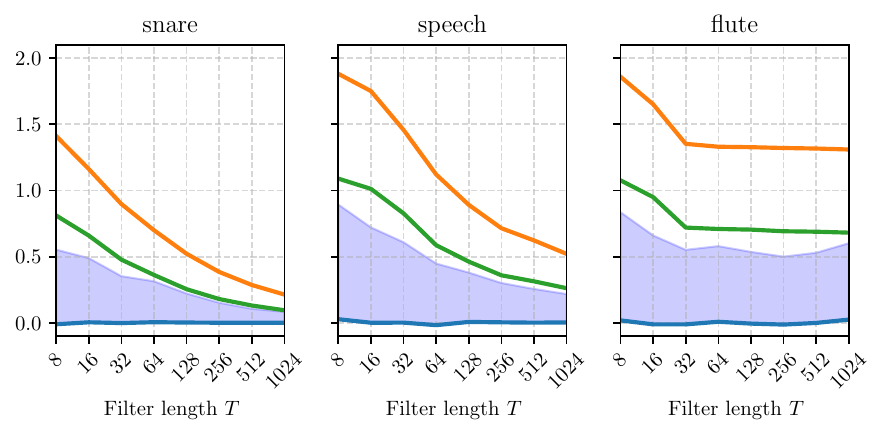}
    \caption{
    Large deviations of filterbank response energy ($\Vert \myPhi\myx \Vert^2 -\Vert \myx \Vert^2$) for three synthetic signals of length $N=1024$ (top) and three natural signals of length $N=22050$ (bottom).
    Blue: empirical mean and $95^{\textrm{th}}$ percentile across $1000$ realizations of $\myPhi$.
    We show two theoretical bounds from Proposition \ref{prop:cheb}: Cantelli (Equation \eqref{eq:cheb}, orange) and Chernoff (Equation \ref{eq:cher}, green).
    Each filterbank contains $J=10$ filters of length $T=2^k$ where $3\leq k\leq10$. 
    }
    \label{fig:cheb}
\end{figure}

After scaling $\myPhi$ such that it preserves energy on average, i.e. $\mathbb{E}\left[\Vert\myPhi \myx\Vert^2\right] = \Vert\myx\Vert^2$, we now derive upper bounds on the probability of large deviations of $\Vert \myPhi \myx\Vert^2$ given $\myx\neq 0$.
\begin{prop}[Cantelli bound]\label{prop:cheb}
Let $\myPhi$ be a random filterbank with $J$ i.i.d. filters $\myw_{\!j}\sim \mathcal{N}(0,\sigma^2 I)$ of length $T$ and $\sigma^{2} = (JT)^{-1}$. Given a deviation $\alpha\geq0$, the probability of $\Vert \myPhi\myx\Vert^2$ exceeding $(1+\alpha)\Vert \myx\Vert^2$ is bounded from above as
\begin{equation}
\mathbb{P}\big[
\Vert \myPhi \myx \Vert^2 \geq (1+\alpha) \Vert \myx \Vert^2 \big] \leq
\frac{\mathbb{V}\left[\Vert \myPhi\myx\Vert^2\right]}{\mathbb{V}\left[\Vert \myPhi\myx\Vert^2\right]+\alpha^2\myRxx(0)^2}.
\label{eq:cheb}
\end{equation}
\end{prop}

\begin{prop}[Chernoff bound]\label{prop:chernoff}
Let $\mylambda$ denote the vector of eigenvalues of $\myQTx$.
Under the same assumptions as Proposition \ref{prop:cheb}, and given a deviation $\alpha\geq0$, the probability of $\Vert \myPhi\myx\Vert^2$ exceeding $(1+\alpha)\Vert \myx\Vert^2$ is bounded from above as
\begin{equation}
\mathbb{P}\big[
\Vert \myPhi \myx \Vert^2 \geq (1+\alpha) \Vert \myx \Vert^2 \big] \leq
\exp \left(-\dfrac{\alpha^2 JT^{2} \Vert \myx \Vert^4}{2 \alpha T \Vert \mylambda \Vert_{\infty} \Vert \myx \Vert^2 + 2 \Vert \mylambda \Vert_2^2}\right).
\label{eq:cher}
\end{equation}
\end{prop}
The two propositions above have their own merits. Proposition \ref{prop:cheb} as direct consequence of Cantelli's inequality \cite{feller68introtoprob} is straightforward and interpretable in terms of the autocorrelation of $\myx$.
Meanwhile, Proposition \ref{prop:chernoff}, based on Chernoff's inequality \cite{chernoff52cher}, is closer to the empirical percentiles, yet is expressed in terms of the eigenvalues of $\myQTx$, for which there is no general formula. The proof is more technical and involves ideas from \cite{birge98lemma} and \cite{laurent00massart}.
Note that in the particular case of full-length filters ($T=N$), $\myQTx$ is a circulant matrix: hence, we interpret these eigenvalues as the energy spectral density of the input signal, i.e., $\mylambda=\vert \hat{\myx}\vert^2$ where $\hat{\myx}$ is the discrete Fourier transform of $\myx$. Detailed proofs for both propositions can be found in the appendix.

\subsection{Numerical simulation}\label{sec:x-known,exp}
We compute empirical probabilities of relative energy deviations between $\Vert \myPhi\myx \Vert^2$ and $\Vert \myx \Vert^2$ for different signals $\myx$ and filter lengths $T$.
Specifically, for each $\myx$ and each $T$, we simulate 1000 independent realizations of $\Vert \myPhi\myx \Vert^2$ for each value of $T$ and retain the closest 95\% displayed as shaded area in Figure \ref{fig:cheb}.
Additionally, we set the right-hand side of Propositions \ref{prop:cheb} and \ref{prop:chernoff} to $5\%$ and solve for $\alpha$, yielding upper bounds for this area.

The upper part of Figure \ref{fig:cheb} illustrates our findings for three synthetic signals: $(i)$ a single impulse, which has low autocorrelation, $(ii)$ a realization of Brownian noise, which has medium autocorrelation and $(iii)$ a sine wave with frequency $\omega = \pi$, which has high autocorrelation.
In the lower part of the same figure, we use real-world sounds: a snare drum hit, a spoken utterance, and a sustained note on the concert flute.

As predicted by the theory, large deviations of $\Vert \myPhi\myx \Vert^2$ become less probable as the filters grow in length $T$ if the input $\myx$ has little autocorrelation (e.g., snare). The rate of decay is slower for highly autocorrelated signals (e.g., flute). These findings explain the observations we already made in Figure \ref{fig:spec}.

\section{Extreme Value Theory meets Frame Theory}
\label{sec:x-unknown}

In the previous section, we have described the probability distribution of $\Vert \myPhi\myx \Vert^2$ for a known input signal $\myx$.
We now turn to inquire about the properties of $\myPhi$ as a linear operator; i.e., independently of $\myx$.
If there exist two positive numbers $A$ and $B$ such that the double inequality $
    A \Vert \myx \Vert^2 \leq
\Vert \myPhi \myx \Vert^2 \leq
B \Vert \myx \Vert^2$
holds for any $\myx \in \RR^N$, $\myPhi$ is said to be a \textit{frame} for $\mathbb{R}^N$ with frame bounds $A$ and $B$. The optimal frame bounds are given by $A = \min_{\Vert \myx \Vert_2 = 1} \Vert\myPhi\myx\Vert^2$ and $ B = \max_{\Vert \myx \Vert_2 = 1} \Vert\myPhi\myx\Vert^2.$

\begin{figure}
    \centering
    \includegraphics[width =\linewidth]{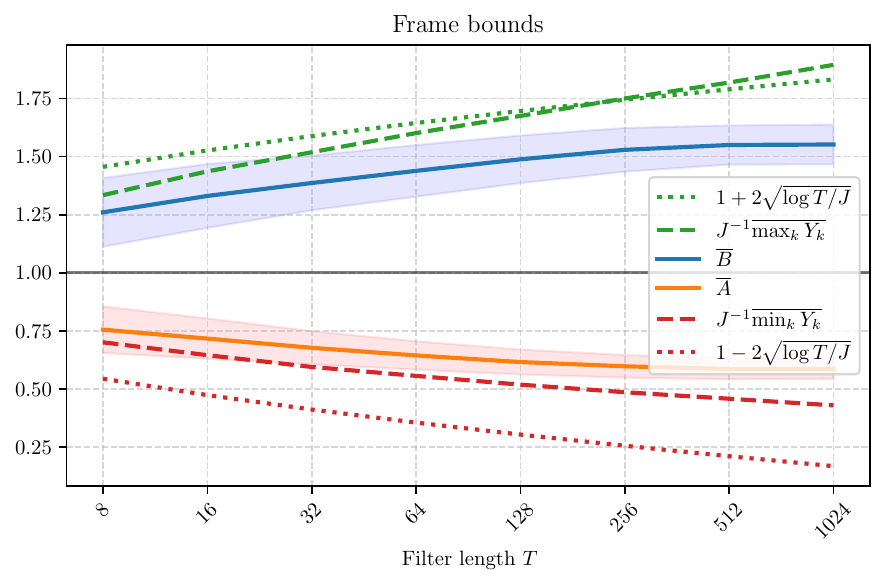}
    \caption{Empirical means $\overline{A}$ and $\overline{B}$ (solid lines) and $95^{\mathrm{th}}$ percentiles (shaded area) of frame bounds $A$ and $B$ for $1000$ instances of $\myPhi$ with $\sigma^2=(TJ)^{-1}$, $J=40$ and different values of $T$.
    Dashed lines denote the bounds of $\mathbb{E}[A]$ and $\mathbb{E}[B]$ from Theorem \ref{thm:expected-frame-bounds}.
    Dotted lines denote the asymptotic bounds in \eqref{eq:kappa_tilde}.}
    \label{fig:expected_fb}
\end{figure}

\subsection{From quadratic forms to chi-squared distributions}
Although the expected frame bounds $\mathbb{E}[A]$ and $\mathbb{E}[B]$ do not have closed-form expressions, we can relate them to the expected order statistics of the chi-squared distribution with $J$ degrees of freedom, denoted by $\chi^2(J)$.

\begin{thm}
\label{thm:expected-frame-bounds}
Let $\myPhi$ be a random filterbank with $J$ i.i.d. filters $\myw_{\!j}\sim \mathcal{N}(0,\sigma^2 I)$ with $\sigma^{2} = \left(JT\right)^{-1}$.
The expectations of the optimal frame bounds $A,B$ of $\myPhi$ are bounded by the order statistics of $Y_0, \ldots, Y_{T-1} \sim \chi^2(J)$ i.i.d., as follows
\begin{equation}\label{eq:expAB}
J^{-1} \mathbb{E}[Y^{\min}_{T}]
\leq
\mathbb{E}\left[A\right]
\leq
1
\leq
\mathbb{E}\left[B\right]
\leq
J^{-1} \mathbb{E}\left[Y^{\max}_{T}\right], 
\end{equation}
where $Y^{\min}_{T} = \min_{0\leq k < T} Y_k$ and $Y^{\max}_{T} = \max_{0\leq k < T} Y_k.$
\end{thm}

\begin{proof}
The inner inequalities $\left(\mathbb{E}\left[A\right]\leq 1 \leq \mathbb{E}\left[B\right]\right)$ are a direct consequence of Proposition \ref{prop:exp}.
Regarding the outer inequalities, we perform an eigenvalue decomposition of $\myQTx = \myU\Lambda\myU^{\top},$ where the columns of $\myU$ contain the eigenvectors of $\myQTx$ as columns and the diagonal matrix $\Lambda$ contains the spectrum of eigenvalues, $\lambda$.
For each filter $\myw_{\!j}$ with $1\leq j \leq J$, let us use the shorthand $\myy_j = U^\top \myw_{\!j}$.
By Lemma \ref{lem:quad} we obtain
\begin{align}
\Vert \myx \ast \myw_{\!j} \Vert^2 &=
\myw_{\!j}^\top U^{\top} \Lambda U \myw_{\!j}
= \sum_{k=0}^{T-1} \lambda_k \myy_{j}[k]^2.
\label{eq:quadr-diag1}
\end{align}
We define $Y_k = \sum_{j=1}^{J} (\myy_{j}^2[k]/\sigma^2)$.
Equation \eqref{eq:quadr-diag1} yields
\begin{align}
    \Vert\myPhi \myx\Vert^2
    = \sigma^2 \sum_{k=0}^{N-1} \lambda_k \sum_{j=1}^{J} \frac{\myy_{j,k}^2}{\sigma^2}
    = \sigma^2 \sum_{k=0}^{N-1} \lambda_k Y_k.
    \label{eq:quadr-diag2}
\end{align}
Since $\myQTx$ is a real symmetric matrix, $\myU$ is an orthogonal matrix.
Thus, $\myy_j$ follows the same distribution as $\myw_{\!j}$
\begin{equation}
\myU^\top \myw_{\!j}\sim \mathcal{N}(0,\sigma^2 \myU I \myU^\top ) = \mathcal{N}(0,\sigma^2 I).
\end{equation}
For all $k$ with $0 \leq k < T$, $\myy_{j}[k]/\sigma^2$ are i.i.d. standard Gaussian random variables.
Thus, the $Y_k$'s are also i.i.d. and follow a $\chi^2(J)$ distribution.
Let us define the associated order statistics
\begin{equation}
    Y^{\min}_{T} = \min_{0\leq k < T} Y_k\quad\textrm{ and }\quad
    Y^{\max}_{T} = \max_{0\leq k < T} Y_k.
    \label{eq:Ymin-Ymax}
\end{equation}
Lemma \ref{lem:quad} implies $\sum_{k=0}^{T-1} \lambda_k = \Trace \left(\myQTx\right) = T \|\myx\|^2$.
Hence
\begin{align}\label{eq:Y_T}
\begin{split}
        \min_{\Vert \myx \Vert_2 = 1} \Vert \myPhi\myx \Vert^2 - \sigma^2 T Y_{T}^{\min} &\geq 0,\\
    \max_{\Vert \myx \Vert_2 = 1} \Vert \myPhi\myx \Vert^2 - \sigma^2 T Y_{T}^{\max} &\leq 0,
\end{split}
\end{align}
where the inequalities are understood as almost sure. Taking the expectation and setting $\sigma^2 = (JT)^{-1}$ yields the claim.
\end{proof}

The numerical simulations in Figure \ref{fig:expected_fb} align well with the statement of Theorem \ref{thm:expected-frame-bounds}.
We observe that optimal frame bounds $A$ and $B$ typically diverge away from one as $T$ grows up to $2^{10}$, a common value in audio applications.
This phenomenon is evidence of instabilities of a convnet at initialization. Our preliminary experiments showed that these instabilities are not compensated during training. Yet, further examination is needed to formulate a rigorous statement here. 

To bound the variances of $A$ and $B$, we use that the extreme values in \eqref{eq:var} are attained for an impulse and a constant signal,  respectively. A proof can be found in the appendix.
\begin{prop}\label{prop:varAB}
Let $\myPhi$ be a random filterbank with $J$ i.i.d. filters $\myw_{\!j}\sim \mathcal{N}(0,\sigma^2\ I)$ with $\sigma^{2} = \left(JT\right)^{-1}$. The variances of the optimal frame bounds $A$ and $B$ can be bounded as
\begin{equation}\label{eq:varAB}
2(TJ)^{-1} \leq \mathbb{V}\left[A\right],\mathbb{V}\left[B\right] \leq 2J^{-1}.
\end{equation}
\end{prop}

\subsection{Asymptotics of the condition number}
The ratio $\kappa = B/A$, known as condition number, characterizes the numerical stability of $\myPhi$. In particular, $\kappa$ equals one if and only if there exists $C>0$ such that $\Vert \myPhi \myx \Vert^2 = C \Vert \myx \Vert^2$. However, its expected value, $\mathbb{E}\left[\kappa\right]$, may be strictly greater than one even so $\mathbb{E}\left[\Vert \myPhi \myx \Vert^2\right] = C \Vert \myx \Vert^2$ holds for every $\myx$. Since $A$ and $B$ are dependent random variables, $\mathbb{E}\left[\kappa\right]$ is difficult to study analytically \cite{vargas22circcond}.
We conjecture that $\mathbb{E}[\kappa] \leq (\mathbb{E}[B]/\mathbb{E}[A])$, which is equivalent to $\mathrm{cov}(\kappa, A) \geq 0$.

Unfortunately, the expected values of $Y_{T}^{\min}$ and $Y_{T}^{\max}$ that are used for the bounds of $\mathbb{E}\left[A\right]$ and $\mathbb{E}\left[B\right]$ in Theorem \ref{thm:expected-frame-bounds} are not available in closed form for finite values of $T$ \cite{casella2021chisquared}.
Nevertheless, for a large number of degrees of freedom $J$, $\chi^2(J)$ resembles a normal distribution with mean $J$ and variance $2J$, such that we propose to replace $Y_{T}^{\min}$ and $Y_{T}^{\max}$ by
\begin{equation}\label{eq:tildeY}
\tilde{Y}_{T}^{\min} =
\min_{0\leq k<T}
\tilde{Y}_k
\quad\textrm{and}\quad
\tilde{Y}_{T}^{\max} =
\max_{0\leq k<T}
\tilde{Y}_k,
\end{equation}
where the $\tilde{Y}_k$'s are i.i.d. drawn from $\mathcal{N}(J, 2J)$ \cite{lehmann2006testing}.
From the extreme value theorem for the standard normal distribution (see e.g. Theorem 1.5.3. in \cite{leadbetter1983extremes}) we know that for large $T$, we can asymptotically approximate the expectations of \eqref{eq:tildeY} by
\begin{equation}
    \mathbb{E}\left[\tilde{Y}^{\min}_T\right]\propto J-2\sqrt{J\log T}\quad \text{and}\quad \mathbb{E}\left[\tilde{Y}^{\max}_T\right]\propto J+2\sqrt{J\log T}.
\end{equation}
The equations above suggest approximate bounds for $\mathbb{E}[A]$ and $\mathbb{E}[B]$.
We draw inspiration from them to propose the value
\begin{equation}\label{eq:kappa_tilde}
\tilde{\kappa}(J,T) =\left(1 + 2\sqrt{\dfrac{ \log T}{J}}\right)\Bigg{/}\left(1 -2\sqrt{\dfrac{ \log T}{J}}\right),
\end{equation}
as asymptotic error bound for $\mathbb{E}[\kappa]$, subject to $T\rightarrow\infty$ and $J>4\log T$.
Interestingly, the level sets of $\tilde{\kappa}$ satisfy $J \propto \log T$, a scaling law which is reminiscent of the theory underlying the construction of discrete  wavelet bases \cite{mallat08wavelettour}.

\subsection{Numerical simulation}
Figure \ref{fig:cond} (top) shows empirical means of $\kappa$ for $1000$ independent realizations of $\myPhi$ and various settings of $J$ and $T$.
Qualitatively speaking we observe that convnets with few long filters (small $J$, large $T$) suffer from ill-conditioning, as measured by a large $\kappa$.
Figure \ref{fig:cond} (bottom) shows the result of the same simulation with $J$ on the horizontal axis, together with our proposed scaling law $J\propto\log T$. We observe that filterbanks that follow this scaling law have approximately the same condition number $\kappa$ on average.

\section{Conclusion}
This article presents large deviation formulas of energy dissipation in random filterbanks.
We have found that the variance of output energy $\Vert \myPhi\myx \Vert^2$ grows with the autocorrelation of the input $\myx$.
Thus, natural audio signals, which typically have high short-term autocorrelation, are \emph{adversarial examples} to 1-D convnets and trigger numerical instabilities with high probability.
Furthermore, we have shown that numerical stability depends strongly on the number of filters $J$ and their lengths $T$,
and that convnets are most stable with many short filters.
For large convnets, we have identified a scaling law ($J\propto\log T$) which roughly preserves the condition number of $\myPhi$.
Characterizing the probability distribution of the condition number for non-asymptotic values of $J$ and $T$ remains an open problem.
In practice, our findings motivate the use of regularization mechanisms that compensate for large autocorrelation of audio data, e.g., by adding adaptive noise.
To be able to draw conclusions from the instabilities at initialization to instabilities during training, further investigations of the effects of gradient descent in this setting are necessary.\footnote{
The source code for reproducing all numerical simulations can be found under \href{https://github.com/danedane-haider/Random-Filterbanks}{https://github.com/danedane-haider/Random-Filterbanks}.}

\begin{figure}
    \centering
    \includegraphics[width =\linewidth]{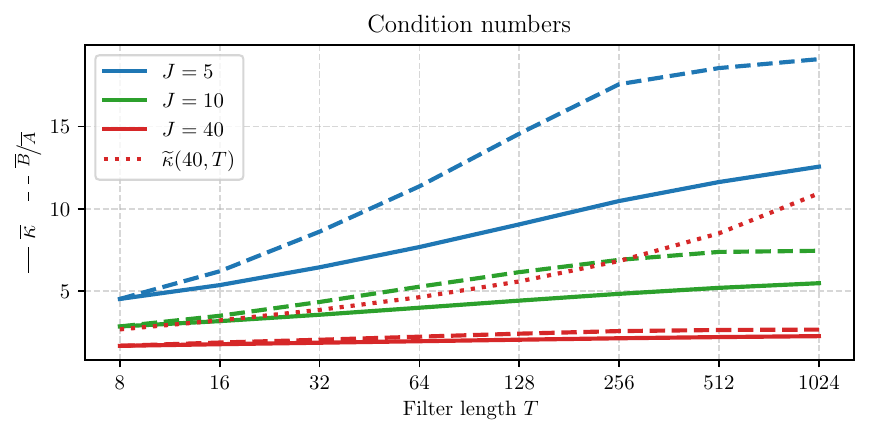}
    \label{fig:cond_naked}
    \includegraphics[width =\linewidth]{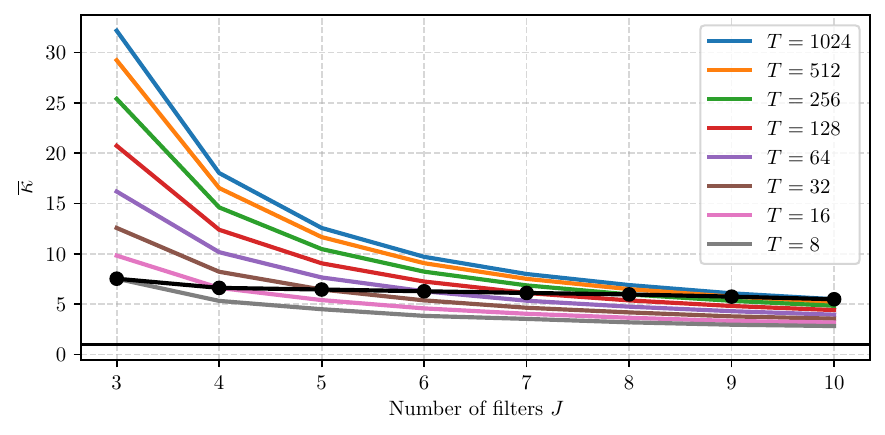}
    \caption{We denote by $\overline{A},\overline{B}$, and $\overline{\kappa}$ the empirical means of the respective quantities over $1000$ instances of $\myPhi$ with $\sigma^2=(TJ)^{-1}$.
    Top: Comparison of $\overline{\kappa}$ (solid) and $\overline{B}/\overline{A}$ (dashed) for increasing filter length $T$ and different values of $J$.
    Bottom: Empirical mean $\overline{\kappa}$ for increasing numbers of filters $J$ and different values $T$. For $J=\log_2 T$ (solid black), $\overline{\kappa}$ remains approximately constant.}
    \label{fig:cond}
\end{figure}

\section*{Acknowledgment}
D. Haider is recipient of a DOC Fellowship of the Austrian Academy of Sciences at the Acoustics Research Institute (A 26355). V. Lostanlen is supported by ANR MuReNN. The work of M. Ehler was supported by the WWTF project CHARMED (VRG12-009) and P. Balazs was supported by the FWF projects LoFT (P 34624) and NoMASP (P 34922).

\clearpage
\bibliographystyle{IEEEtran}
\bibliography{IEEEabrv,lostanlen2023spl_frames}

\clearpage
\thispagestyle{empty} 
\section{Appendix}
As a complement to what we derived for convnets in Proposition \ref{prop:exp}, we show that the variance of the energy of a fully--connected layer with Gaussian initialization does not depend on the characteristics of the input signal. To see this, we use that any Gaussian matrix $\myW\in \RR^{M\times N}$ with $ M\geq N$ is associated to a random tight frame of any order $p$, i.e., there is $C_p>0$ such that $\mathbb{E}\left[\Vert \myW\myx\Vert^{2p}\right] = C_p \Vert\myx\Vert^{2p}$ for any $p>1$ \cite{ehler2015pre}. For mean zero and variance $\sigma^2$ we have that $C_p=M(M+2)\cdots(M+2p-2)\sigma^{2p}$, see Example 4.4 in \cite{ehler2015pre}.
\begin{prop}\label{prop:gauss}
    Let $\myx\in \RR^N$ and $\myW\in \RR^{M\times N},\; M\geq N$ be a random matrix with entries sampled i.i.d. from $\mathcal{N}(0,\sigma^2)$.
    Then
    \begin{align}\label{eq:gaussexp}
        \mathbb{E}\left[\Vert \myW\myx\Vert^2\right] &= M \sigma^2 \Vert\myx\Vert^2,\\ 
        \mathbb{V}\left[\Vert \myW\myx\Vert^2\right] &= 2M\sigma^4 \Vert\myx\Vert^4.\label{eq:gaussvar}
    \end{align}
\end{prop}

\begin{proof}
For $p=1$, we have that $C_1=M\sigma^{2}$, showing \eqref{eq:gaussexp}.
For the variance, we use that $C_2=M(M+2)\sigma^4$ and deduce
    \begin{align*}
    \mathbb{V}\left[\Vert \myW\myx\Vert^2\right] 
    &= \mathbb{E}\left[ \left( \Vert \myW\myx\Vert^2 - M\sigma^2 \Vert\myx\Vert^2 \right)^2 \right]\\
    &= \mathbb{E}\left[ \Vert \myW\myx\Vert^4 \right] - M^2\sigma^4\Vert\myx\Vert^4
    = 2M\sigma^4 \Vert\myx\Vert^4.
    \end{align*}
\end{proof}
By Proposition \ref{prop:exp}, a random filterbank $\myPhi$ is a random tight frame of order one. For $p>1$, this is in general not the case.

\begin{proof}[Proof of Lemma \ref{lem:quad}]
Given $\myx\in \RR^N$ and $\myw\in \RR^T$, we write the circulant convolution $\myx \ast \myw$ in Equation \eqref{eq:conv} as the matrix-vector multiplication $C_T(\myx) \myw$ where
\begin{equation*}
    C_T(\myx) =
    \begin{pmatrix}
        \ \myx[0] & \myx[N-1] & \cdots & \myx[N-T+1]\ \ \\
        \ \myx[1] & \myx[0] & \cdots & \myx[N-T+2]\ \ \\
        \ \vdots & \vdots &  & \vdots\ \  \\
        \ \myx[N-2] & \myx[N-3] & \cdots & \myx[N-T-1]\ \  \\
        \ \myx[N-1] & \myx[N-2] & \cdots & \myx[N-T]\ \ 
    \end{pmatrix}
\end{equation*}
contains the first $T$ columns of the circulant matrix generated by a reversed version of $\myx$.
The entries are given by $$C_T(\myx)[n,t]=\myx[(n-t)\;\mathrm{mod}\;N]$$ for $0 \leq n < N$ and $0 \leq t < T$. 
We write down its squared Euclidean norm as a quadratic form
\begin{align*}
    \Vert\myx \ast \myw\Vert^2
    = \langle C_T(\myx) \myw,C_T(\myx) \myw \rangle
    = \langle\myw, Q_T(\myx) \myw\rangle
\end{align*}
where $Q_T(\myx) = C_T(\myx)^\top C_T(\myx)$.
Recalling the definition of circular autocorrelation \eqref{eq:auto}, we conclude with
\begin{align*}
    Q_T(\myx)[t, t'] &= \sum_{n = 0}^{N-1} \myx[(n-t) \;\mathrm{mod}\;N]\ \myx[(n-t') \;\mathrm{mod}\;N]\\
    &= R_{\myx \myx}((t'-t)\;\mathrm{mod}\;N).
\end{align*}
The moreover part is easily seen by $0\leq t < T$,
\begin{equation}
Q_{T}(\myx)[t, t] =
R_{\myx \myx}(0) =
\sum_{n=0}^{N-1}
\myx[n]^2
= \Vert \myx \Vert^2.
\end{equation}
\end{proof}

\begin{proof}[Proof of Proposition \ref{prop:cheb}]
We recall Cantelli's inequality \cite{feller68introtoprob}:
\begin{equation}
\mathbb{P}\Big[
 Z - \mathbb{E}\left[Z\right] \geq \beta
\Big]
\leq \dfrac{\mathbb{V}\left[Z\right]}{\mathbb{V}\left[Z\right]+\beta^2}.
\end{equation}
where $\beta>0$ and $Z$ has finite mean and variance.
Given $\alpha$ and $\myx$, we set $Z=\Vert\myPhi\myx\Vert^2$ and $\beta = \alpha \Vert \myx \Vert^2$.
With Proposition \ref{prop:exp}, we replace $\mathbb{E}[Z]$ by $JT \sigma^2\Vert \myx \Vert^2$.
With Lemma \ref{lem:quad}, we replace $\Vert \myx \Vert^4$ by $\myRxx[0]^2$.
Setting $\sigma^2=(JT)^{-1}$ concludes the proof.
\end{proof}

Our proof of Proposition \ref{prop:chernoff} hinges on the following lemma.
\begin{lem}[Lemma 8 in Birgé \emph{et al.} \cite{birge98lemma}]\label{lem:birge}
For any $v,c,\beta >0$,
$$\inf_{\mu>0} \frac{\mu^2v^2}{1-2\mu c} -\mu\beta\leq -\frac{\beta^2}{2c\beta + 2v^2}.$$
\end{lem}

\begin{proof}[Proof of Proposition \ref{prop:chernoff}]
We show \eqref{eq:cher} via the generic Chernoff bounds for any random variable $Z$
\begin{equation}
\mathbb{P}\left[
Z \geq \beta
\right]
\leq \inf_{\mu>0} \mathbb{E}\left[e^{\mu Z}\right]e^{-\mu \beta}.
\end{equation}
We set $Z=\Vert \myPhi \myx \Vert^2- \Vert \myx \Vert^2$ and use \eqref{eq:quadr-diag2}, together with Lemma \ref{lem:quad} to see that $Z=\sum_{k=0}^{T-1}\sum_{j=1}^J\sigma^2\lambda_k(\myy_{j}[k]^2-1)$. A straightforward computation gives
\begin{align*}
    \log \mathbb{E}\left[ e^{\mu Z} \right] = \sum_{k=0}^{T-1}\sum_{j=1}^J \log \mathbb{E}\left[ \exp\left(\mu\sigma^2\lambda_k(\myy_{j}[k]^2-1)\right) \right].
\end{align*}
Recall that $\frac{\myy_{j}[k]}{\sigma^2}\sim \mathcal{N}(0,1)$. Analog to the proof of Lemma $1$ in \cite{laurent00massart}, we use that the mapping $\psi:u\mapsto \log \mathbb{E}\left[ \exp\left(u\sigma^2( X^2-1)\right) \right]$ satisfies
$\psi(u)\leq \frac{u^2\sigma^4}{1-2u\sigma^2}$ for any $X\sim \mathcal{N}(0,1)$ and $0<u<\frac{1}{2\sigma^2}$. Since $\myCTx$ is a principal sub-matrix of a positive definite matrix (autocorrelation matrix), $\lambda_k>0$ for all $k=0,\dots, T-1$. Therefore, for $\mu<\frac{1}{2\sigma^2\max_k\lambda_k}$,
\begin{align}\label{eq:log}
    \log \mathbb{E}\left[ e^{\mu Z} \right]
    \leq
    \sum_{k=0}^{T-1}\sum_{j=1}^J \frac{\left(\mu \lambda_k\right)^2\sigma^4}{1-2\mu\sigma^2 \lambda_k}
    \leq \frac{\mu^2\sigma^4J\Vert \mylambda \Vert_2^2}{1-2\mu\sigma^2\Vert \mylambda \Vert_\infty}.  
\end{align}
Finally, using \eqref{eq:log} and Lemma \ref{lem:birge} with $v^2 = \sigma^4J\Vert \mylambda \Vert_2^2$ and $c=\sigma^2\Vert \mylambda \Vert_\infty$, we obtain
\begin{align*}
    \inf_{\mu>0} \mathbb{E}\left[e^{\mu Z}\right]e^{-\mu \beta}
    &= \exp\left( \inf_{\mu>0} \log \mathbb{E}\left[e^{\mu Z}\right] -\mu \beta \right)\\
    &\leq \exp\left( \inf_{\mu>0} \frac{\mu^2\sigma^4J\Vert \mylambda \Vert_2^2}{1-2\mu\sigma^2\Vert \mylambda \Vert_\infty} -\mu\beta \right)\\    
    &\leq \exp\left(- \frac{\beta^2}{2\beta \sigma^2 \Vert \mylambda \Vert_\infty + 2\sigma^4 J \Vert \mylambda \Vert_2^2} \right).
\end{align*}
Setting $\beta = \alpha \Vert \myx \Vert^2$ and $\sigma^2=(JT)^{-1}$ yields the claim.
\end{proof}

\begin{proof}[Proof of Proposition \ref{prop:varAB}]
    Observe that $$\min_{\Vert \myx \Vert^2=1}R_{\myx\myx}(t)^2=\begin{cases}
        1 &\text{ if } t=0\\ 0 &\text{ otherwise.}
    \end{cases}\quad \text{ and }\quad \max_{\Vert \myx \Vert^2=1}R_{\myx\myx}(t)^2=1$$
    for $0\leq t <T$. These extreme values are attained for an impulse and a constant signal respectively.
    Using these signals in Equation \eqref{eq:var} of Proposition \ref{prop:exp} and setting $\sigma^2=(TJ)^{-1}$ yields the result.
\end{proof}

\end{document}